\documentclass[11pt, a4paper]{article}

\usepackage[utf8]{inputenc}
\usepackage[T1]{fontenc}
\usepackage{geometry}
\usepackage{tikz}
\usetikzlibrary{arrows.meta}

\usepackage{amsmath, amssymb, amsthm, mathtools}
\usepackage{bm}

\usepackage{algorithm}
\usepackage{algpseudocode}
\usepackage{listings}
\usepackage{xcolor}

\usepackage[colorlinks=true, linkcolor=blue, citecolor=teal, urlcolor=cyan]{hyperref}

\newtheorem{theorem}{Theorem}
\newtheorem{lemma}{Lemma}

\theoremstyle{definition}

\theoremstyle{remark}

\definecolor{codegreen}{rgb}{0,0.6,0}
\definecolor{codegray}{rgb}{0.5,0.5,0.5}
\definecolor{codepurple}{rgb}{0.58,0,0.82}
\definecolor{backcolour}{rgb}{0.95,0.95,0.92}

\lstdefinestyle{mystyle}{
    backgroundcolor=\color{backcolour},   
    commentstyle=\color{codegreen},
    keywordstyle=\color{magenta},
    numberstyle=\tiny\color{codegray},
    stringstyle=\color{codepurple},
    basicstyle=\ttfamily\footnotesize,
    breakatwhitespace=false,         
    breaklines=true,                 
    captionpos=b,                    
    keepspaces=true,                 
    numbers=left,                    
    numbersep=5pt,                  
    showspaces=false,                
    showstringspaces=false,
    showtabs=false,                  
    tabsize=2
}
\title{\textbf{The Geometry of the Pivot:} \\ A Note on Lazy Pivoted Cholesky and Farthest Point Sampling}

\author{
    \textbf{Gil Shabat} \\
    \texttt{gil.shabat@cs.tau.ac.il}
}

\date{\today}

\begin{document}

\maketitle

\begin{abstract}
Low-rank approximations of large kernel matrices are ubiquitous in machine learning, particularly for scaling Gaussian Processes to massive datasets. The Pivoted Cholesky decomposition is a standard tool for this task, offering a computationally efficient, greedy low-rank approximation. While its algebraic properties are well-documented in numerical linear algebra, its geometric intuition within the context of kernel methods often remains obscure. In this note, we elucidate the geometric interpretation of the algorithm within the Reproducing Kernel Hilbert Space (RKHS). We demonstrate that the pivotal selection step is mathematically equivalent to \textit{Farthest Point Sampling} (FPS) using the kernel metric, and that the Cholesky factor construction is an implicit Gram-Schmidt orthogonalization. We provide a concise derivation and a minimalist Python implementation to bridge the gap between theory and practice.
\end{abstract}

\section{Introduction}

Kernel methods, such as Gaussian Processes (GPs), provide a flexible framework for non-parametric regression and classification. However, they typically scale poorly with dataset size $N$, requiring $\mathcal{O}(N^3)$ operations for inversion and $\mathcal{O}(N^2)$ for storage of the kernel matrix $K$. To make these methods tractable for large $N$, low-rank approximations $K \approx L L^\top$, where $L \in \mathbb{R}^{N \times M}$ and $M \ll N$, are widely employed \cite{rasmussen2006gaussian}.

Beyond direct approximation for inference, these decompositions play a crucial role as \textbf{preconditioners} for iterative linear solvers. When solving systems $Kx=y$ using methods like Conjugate Gradients (CG), the convergence rate depends on the condition number of $K$. A low-rank Pivoted Cholesky factor (or the related Interpolative Decomposition) is frequently used to construct effective preconditioners that cluster the eigenvalues of the system, enabling rapid convergence even for massive datasets with millions of points \cite{gardner2018gpytorch, cutajar2016preconditioning, shabat2021fast}.

Among various approximation techniques---such as the Nyström method \cite{williams2001using} and Random Fourier Features \cite{rahimi2007random}---the \textbf{Pivoted Cholesky Decomposition} stands out for being parameter-free, deterministic, and amenable to "lazy" evaluation. The lazy variant computes matrix entries only on demand, drastically reducing computational complexity from $\mathcal{O}(N^3)$ to $\mathcal{O}(NM^2)$ and memory complexity from $\mathcal{O}(N^2)$ to $\mathcal{O}(NM)$.

Despite its adoption in modern libraries like GPyTorch \cite{gardner2018gpytorch}, the algorithm is often taught as a purely algebraic manipulation of matrix entries to minimize the trace norm of the error. While the standard Pivoted Cholesky is a classic textbook algorithm \cite{golub2013matrix}, its "lazy" adaptation for kernel matrices---and specifically the explicit connection to geometric sampling---is often treated as folklore in the machine learning community, lacking a definitive primary reference. This note aims to fill this gap. We emphasize that we do not propose a new algorithm; rather, we formalize the geometric intuition commonly shared among practitioners but rarely documented: the algorithm is essentially a greedy volume maximization strategy that performs \textit{Farthest Point Sampling} in the feature space.

\section{The Algorithm: Lazy Pivoted Cholesky}

The goal is to construct a low-rank approximation that minimizes the reconstruction error. A natural objective is to minimize the \textbf{trace} of the residual matrix $E_m = K - L_m L_m^\top$, which measures the total unexplained variance. Ideally, at each step, one would select the pivot that maximizes the reduction in the total trace. This reduction corresponds to the squared Euclidean norm of the new column, $\sum_i L_{i,m}^2$.

However, computing the full column norm for every candidate pivot would require $\mathcal{O}(N^2)$ operations per step, defeating the efficiency of the lazy evaluation. Instead, the algorithm employs a greedy heuristic: it selects the pivot $i^*$ that maximizes the diagonal entry $d[i^*]$.

\begin{equation}
    \text{trace}(E_m) = \sum_{i=1}^N d[i]
\end{equation}

Thus, the algorithm greedily eliminates the largest diagonal error. While not strictly maximizing the total trace reduction (which would require computing all columns), the maximum diagonal entry serves as a tight lower bound for the column energy and acts as an efficient proxy for trace minimization without instantiating the full matrix columns.

\begin{algorithm}
\caption{Lazy Pivoted Cholesky Decomposition}
\label{alg:lazy_cholesky}
\begin{algorithmic}[1]
\State \textbf{Input:} Kernel function $k(\cdot, \cdot)$, Data $X$, Rank $M$
\State Initialize diagonal residuals: $d_i = k(x_i, x_i)$ for all $i=1 \dots N$
\State Initialize factor: $L = \mathbf{0}_{N \times M}$
\State Initialize permutation: $\pi = [1, \dots, N]$
\For{$m = 1$ to $M$}
    \State \textcolor{blue}{\textbf{// 1. Selection Step (Pivoting)}}
    \State \textcolor{gray}{\textit{// Select pivot with maximal residual variance}}
    \State $i^* \gets \arg\max_{j \in \{m, \dots, N\}} d[j]$ 
    \State Swap indices $m$ and $i^*$ in $\pi, d$, and rows of $L$
    
    \State \textcolor{blue}{\textbf{// 2. Basis Construction}}
    \State $L[m, m] \gets \sqrt{d[m]}$
    
    \State \textcolor{blue}{\textbf{// 3. Lazy Column Evaluation}}
    \State Compute raw column: $\mathbf{c} \gets k(X[\pi_{m+1:}], X[\pi_m])$
    \State Update column (Schur complement):
    \State $L[m+1:, m] \gets \frac{1}{L[m,m]} \left( \mathbf{c} - L[m+1:, 1:m-1] L[m, 1:m-1]^\top \right)$
    
    \State \textcolor{blue}{\textbf{// 4. Residual Update}}
    \State $d[m+1:] \gets d[m+1:] - L[m+1:, m]^2$
\EndFor
\State \textbf{Return} $L, \pi$
\end{algorithmic}
\end{algorithm}

\section{Geometric Interpretation in RKHS}

Let $\mathcal{H}$ be the Reproducing Kernel Hilbert Space (RKHS) associated with kernel $k$, and let $\phi: \mathcal{X} \to \mathcal{H}$ be the feature map such that $k(x, y) = \langle \phi(x), \phi(y) \rangle_\mathcal{H}$.

\subsection{The Implicit Basis and QR Connection}
At step $m$, let $S_{m-1}$ be the subspace spanned by the feature vectors of the previously selected pivots:
\begin{equation}
    S_{m-1} = \text{span}\left\{ \phi(x_{\pi_1}), \dots, \phi(x_{\pi_{m-1}}) \right\} \subset \mathcal{H}.
\end{equation}
For the base case $m=1$, we define $S_0 = \{\mathbf{0}\}$, representing the trivial subspace.
The Cholesky algorithm implicitly constructs an orthonormal basis $\{e_1, \dots, e_M\}$ for this sequence of subspaces. This is equivalent to applying the Gram-Schmidt process to the ordered sequence of selected feature vectors.

To see this connection formally, let $\Phi$ denote the (potentially infinite) matrix of feature vectors. The kernel matrix is $K = \Phi \Phi^\top$. Performing a QR decomposition (Gram-Schmidt) on the transpose of the feature matrix yields $\Phi^\top = Q R$, where $Q$ has orthonormal columns (the basis vectors) and $R$ is upper triangular. Substituting this into the kernel definition:
\begin{equation}
    K = \Phi \Phi^\top = (Q R)^\top (Q R) = R^\top Q^\top Q R = R^\top R.
\end{equation}
Since $Q$ is orthonormal ($Q^\top Q = I$). The Cholesky decomposition provides $K = L L^\top$. Due to the uniqueness of the Cholesky factor, we identify $L = R^\top$.
Thus, the entries $L_{i,j}$ represent the coefficients of the feature vectors when projected onto this implicit orthonormal basis:
\begin{equation}
    L_{i,j} = \langle \phi(x_i), e_j \rangle_\mathcal{H}.
\end{equation}

\subsection{The Update Step as Orthogonal Projection}
To see why Equation (5) holds, consider the algebraic update step for the $m$-th column. The entry $L_{i,m}$ is computed as:
\begin{equation}
    L_{i,m} = \frac{1}{L_{mm}} \left( k(x_i, x_{\pi_m}) - \sum_{j=1}^{m-1} L_{i,j} L_{m, j} \right).
\end{equation}
The summation term represents the inner product of the projections onto the existing basis vectors $e_1, \dots, e_{m-1}$. Note that $L_{m,j}$ denotes the $j$-th coefficient of the pivot (which resides in the $m$-th row of $L$ after permutation). The expression inside the parentheses is therefore the inner product between $\phi(x_i)$ and the \textit{residual vector} of the pivot, $r_m = \phi(x_{\pi_m}) - \text{Proj}_{S_{m-1}}(\phi(x_{\pi_m}))$.

Note that by definition of orthogonal projection, the residual vector $r_m$ is orthogonal to the subspace $S_{m-1}$.
Since the diagonal element serves as the normalization factor $L_{mm} = \sqrt{d[\pi_m]} = \|r_m\|_\mathcal{H}$, the update step explicitly computes the projection onto the new normalized basis vector $e_m = r_m / \|r_m\|_\mathcal{H}$:
\begin{equation}
    L_{i,m} = \frac{\langle \phi(x_i), r_m \rangle_\mathcal{H}}{\|r_m\|_\mathcal{H}} = \langle \phi(x_i), e_m \rangle_\mathcal{H}.
\end{equation}

\subsection{The Residual as Geometric Distance}
The vector $d$ maintains the diagonal of the residual matrix $K - L_m L_m^\top$. Algebraically, the update rule is $d[i] \leftarrow d[i] - L_{i,m}^2$. Geometrically, this is an application of the Pythagorean theorem in $\mathcal{H}$.

The squared norm of any feature vector can be decomposed into the energy contained in the subspace $S_{m-1}$ (the projection) and the energy of the residual vector (the error):
\begin{equation}
    \|\phi(x_i)\|_\mathcal{H}^2 = \| \text{Proj}_{S_{m-1}}(\phi(x_i)) \|_\mathcal{H}^2 + \| \phi(x_i) - \text{Proj}_{S_{m-1}}(\phi(x_i)) \|_\mathcal{H}^2.
\end{equation}
The second term is precisely the squared distance of the point $\phi(x_i)$ from the subspace $S_{m-1}$. Since $\| \text{Proj}_{S_{m-1}}(\phi(x_i)) \|_\mathcal{H}^2 = \sum_{j=1}^{m-1} \langle \phi(x_i), e_j \rangle_\mathcal{H}^2 = \sum_{j=1}^{m-1} L_{i,j}^2$, the value stored in $d[i]$ at step $m$ corresponds to this residual energy:
\begin{equation}
    d[i] = \|\phi(x_i)\|_\mathcal{H}^2 - \sum_{j=1}^{m-1} L_{i,j}^2 = \| \phi(x_i) - \text{Proj}_{S_{m-1}}(\phi(x_i)) \|_\mathcal{H}^2.
\end{equation}

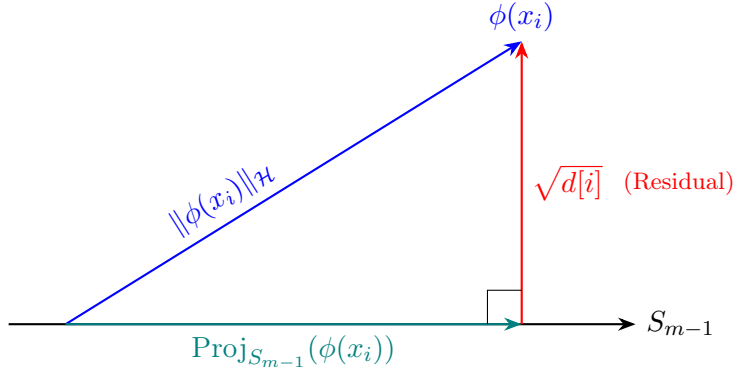
\begin{figure}[ht]
    \centering
    \begin{tikzpicture}[scale=1.5, >=Stealth]
        \coordinate (O) at (0,0);
        \coordinate (P) at (4,0); 
        \coordinate (X) at (4,2.5); 
        
        \draw[thick, ->] (-0.5,0) -- (5,0) node[right] {$S_{m-1}$};
        
        \draw[thick, ->, blue] (O) -- (X) node[midway, above left, sloped] {$\|\phi(x_i)\|_\mathcal{H}$};
        \node[blue, above] at (X) {$\phi(x_i)$};
        
        \draw[thick, ->, teal] (O) -- (P) node[midway, below] {$\text{Proj}_{S_{m-1}}(\phi(x_i))$};
        
        \draw[thick, ->, red] (P) -- (X);
        \node[red, right] at (4, 1.25) {$\sqrt{d[i]}$};
        \node[red, right, font=\footnotesize] at (4.8, 1.25) {(Residual)};
        
        \draw (3.7,0) -- (3.7,0.3) -- (4,0.3);
    \end{tikzpicture}
    \caption{Geometric interpretation of the residual update. The squared norm of the feature vector $\phi(x_i)$ is decomposed via the Pythagorean theorem into the squared norm of its projection onto the current subspace $S_{m-1}$ and the squared norm of the orthogonal residual. The Pivoted Cholesky algorithm greedily selects the point $x_i$ that maximizes this residual distance (height).}
    \label{fig:geometry}
\end{figure}

\subsection{Pivoting is Farthest Point Sampling}
The pivoting step selects the index $i^*$ that maximizes the diagonal entry:
\begin{equation}
    i^* = \arg\max_{i} d[i] = \arg\max_{i} \| \phi(x_i) - \text{Proj}_{S_{m-1}}(\phi(x_i)) \|_\mathcal{H}^2.
\end{equation}
\textbf{Conclusion:} The Pivoted Cholesky algorithm is mathematically equivalent to Greedy Farthest Point Sampling (FPS) \cite{gonzalez1985clustering} in the RKHS. In each iteration, it selects the data point that is "least explained" by the current approximation, i.e., the point farthest from the linear span of the previously selected points.

\subsection{Formal Equivalence}
We formalize the preceding derivation with the following lemma and theorem.

\begin{lemma}[Residual Identity]
At step $m$, the diagonal element $d[i]$ maintained by Algorithm \ref{alg:lazy_cholesky} satisfies:
\begin{equation}
    d[i] = \|\phi(x_i) - \text{Proj}_{S_{m-1}}(\phi(x_i))\|_\mathcal{H}^2,
\end{equation}
where $S_{m-1} = \text{span}\{\phi(x_{\pi_1}), \dots, \phi(x_{\pi_{m-1}})\}$ and $S_0 = \{\mathbf{0}\}$.
\end{lemma}

\begin{proof}
We proceed by induction. For $m=1$, the subspace $S_0$ is the zero vector. Thus, $d[i] = k(x_i, x_i) = \langle \phi(x_i), \phi(x_i) \rangle_\mathcal{H} = \|\phi(x_i)\|_\mathcal{H}^2$, which is the squared distance to the zero vector.
Assume the hypothesis holds for step $m$. The update rule for the diagonal is $d^{(new)}[i] = d^{(old)}[i] - L_{i,m}^2$.
Using the inductive hypothesis and Equation (7), we have:
\begin{equation}
    d^{(new)}[i] = \|\phi(x_i) - \text{Proj}_{S_{m-1}}(\phi(x_i))\|_\mathcal{H}^2 - |\langle \phi(x_i), e_m \rangle_\mathcal{H}|^2.
\end{equation}
Since $e_m$ is orthogonal to $S_{m-1}$ and $\|e_m\|_\mathcal{H}=1$, removing the component projected onto $e_m$ corresponds to computing the squared distance to the larger subspace $S_m$, which is spanned by adding $e_m$ to $S_{m-1}$ (formally, $S_m = S_{m-1} \oplus \text{span}\{e_m\}$).
\end{proof}

\begin{theorem}[Cholesky is FPS]
The sequence of pivots $\{p_1, \dots, p_M\}$ selected by the Lazy Pivoted Cholesky decomposition is identical to the sequence selected by the greedy Farthest Point Sampling algorithm on the feature vectors in $\mathcal{H}$ (assuming consistent tie-breaking).
\end{theorem}

\begin{proof}
The greedy FPS algorithm selects the point $p_m$ that maximizes the distance to the current subspace: $p_m = \arg\max_i \text{dist}(\phi(x_i), S_{m-1})$.
According to the Lemma, the diagonal entry $d[i]$ computed by the Cholesky algorithm is exactly the square of this distance: $d[i] = \text{dist}^2(\phi(x_i), S_{m-1})$.
Since the distance function is non-negative, maximizing the distance is equivalent to maximizing the squared distance. Thus:
\begin{equation}
    \arg\max_i d[i] = \arg\max_i \|\phi(x_i) - \text{Proj}_{S_{m-1}}(\phi(x_i))\|_\mathcal{H}^2.
\end{equation}
Therefore, both algorithms maximize the same objective function at every step, resulting in an identical sequence of selected points.
\end{proof}

\section{Implementation}

The following Python code implements the algorithm, highlighting the "lazy" nature where only necessary columns of the kernel matrix are computed.

\begin{lstlisting}[language=Python, caption=Python implementation of Lazy Pivoted Cholesky]
import numpy as np

def lazy_pivoted_cholesky(kernel_func, X, max_rank, tol=1e-6):
    """
    Computes the partial pivoted Cholesky decomposition.
    
    Args:
        kernel_func: Function (X1, X2, diag=False) -> Matrix or Vector.
                     Must support efficient diagonal computation (O(N)).
        X: Data matrix (N x D).
        max_rank: Maximum rank of the approximation.
        
    Returns:
        L (N x M), pivots (M)
    """
    N = X.shape[0]
    # d represents the squared distance from the current subspace
    d = kernel_func(X, X, diag=True) 
    L = np.zeros((N, max_rank))
    pivots = np.arange(N)
    
    for k in range(max_rank):
        # 1. FPS Step: Find point farthest from current subspace
        i_star = k + np.argmax(d[k:])
        
        # Swap logic
        pivots[[k, i_star]] = pivots[[i_star, k]]
        d[[k, i_star]] = d[[i_star, k]]
        L[[k, i_star], :] = L[[i_star, k]]
        
        # Check convergence (max residual distance < tolerance)
        if d[k] < tol:
            return L[:, :k], pivots[:k]
            
        # 2. Basis construction (Normalization)
        L[k, k] = np.sqrt(d[k])
        
        # 3. Compute column k (Lazy evaluation)
        # We need <phi(x), phi(pivot)> for all remaining x
        col_raw = kernel_func(X[pivots[k+1:]], 
                              X[pivots[k]].reshape(1, -1)).flatten()
        
        # Implicit Gram-Schmidt: Subtract projection on previous basis vectors
        dot_prod = L[k+1:, :k] @ L[k, :k].T
        L[k+1:, k] = (col_raw - dot_prod) / L[k, k]
        
        # 4. Update distances (Pythagoras theorem)
        # dist_new^2 = dist_old^2 - projection_on_new_axis^2
        d[k+1:] -= L[k+1:, k]**2
        
        # Numerical stability: clamp negative values due to float precision
        d[k+1:] = np.maximum(d[k+1:], 0.0) 

    return L, pivots
\end{lstlisting}

\section{Discussion}
Understanding Pivoted Cholesky as a geometric process clarifies its strengths and weaknesses compared to other low-rank methods. 

\begin{itemize}
    \item \textbf{Efficiency:} Unlike Nyström methods that may require a separate clustering step (e.g., K-Means) to select landmarks, Cholesky performs point selection and matrix decomposition simultaneously in a single pass, requiring $\mathcal{O}(NM^2)$ operations \cite{gardner2018gpytorch}.
    
    \item \textbf{Point-wise vs. Subspace FPS:} It is crucial to distinguish between two common interpretations of "Farthest Point Sampling." Standard geometric FPS typically selects the point maximizing the distance to the \textit{nearest neighbor} in the selected set ($x_{next} = \arg\max_{x} \min_{s \in S} \|x - s\|$). Even if implemented in the feature space using the kernel metric ($d(x,s) = \|\phi(x) - \phi(s)\|_\mathcal{H}$), this approach fundamentally places spheres around selected points. In contrast, Pivoted Cholesky acts as a \textit{Subspace FPS} in the RKHS, selecting the point maximizing the orthogonal distance to the \textit{linear subspace} spanned by the selected feature vectors ($x_{next} = \arg\max_{x} \|\phi(x) - \text{Proj}_{\text{span}(S)}(\phi(x))\|_{\mathcal{H}}$). This distinction is critical: a point can be far from any specific selected point (high pairwise distance) yet linearly dependent on their combination (zero Cholesky residual). Thus, Cholesky maximizes the "volume" spanned by the features rather than merely filling the space. Note, however, that for stationary kernels with very narrow bandwidths (fast decay), the feature vectors become nearly orthogonal. In this specific regime, the distance to the subspace is dominated by the distance to the nearest neighbor, and the two sampling strategies effectively coincide.

    \item \textbf{Volume Maximization and RRQR:} Geometrically, maximizing the distance to the subspace is equivalent to greedily maximizing the volume of the simplex formed by the selected feature vectors (or the determinant of the kernel submatrix). This reveals a strong connection to \textit{Rank-Revealing QR} (RRQR) algorithms. While Pivoted Cholesky implements a greedy strategy (often termed "weak RRQR"), \textit{Strong RRQR} algorithms \cite{gu1996efficient} employ additional column swapping strategies to guarantee tighter error bounds. Strong RRQR thus occupies a theoretical middle ground between the efficient greedy Cholesky and the statistically optimal but expensive Leverage Score sampling.
    
    \item \textbf{Optimality vs. Leverage Scores:} While FPS is efficient and robust for space-filling, it is not statistically optimal for minimizing reconstruction error (Frobenius or spectral norm). Sampling proportional to \textit{Statistical Leverage Scores} \cite{mahoney2011randomized} provides tighter error bounds closer to the optimal SVD, as it prioritizes points that are structurally important (high "influence" on the principal components) rather than just geometrically distant outliers. However, computing exact leverage scores is expensive, making Pivoted Cholesky a pragmatic trade-off.
\end{itemize}

\section*{Acknowledgments}
This note was prepared with the assistance of an AI language model, which helped in structuring the arguments, formatting the \LaTeX{}, and refining the mathematical derivations based on the author's insights.

\end{document}